\documentclass[letterpaper,english]{article}
\usepackage{helvet}
\usepackage{courier}
\usepackage[T1]{fontenc}
\usepackage[latin9]{inputenc}
\usepackage{booktabs}
\usepackage{amsmath}
\usepackage{amsthm}
\usepackage{amssymb}
\usepackage{graphicx}

\makeatletter

\pdfpageheight\paperheight
\pdfpagewidth\paperwidth

\newcommand{\lyxmathsym}[1]{\ifmmode\begingroup\def\b@ld{bold}
  \text{\ifx\math@version\b@ld\bfseries\fi#1}\endgroup\else#1\fi}

\providecommand{\tabularnewline}{\\}

\theoremstyle{plain}
\newtheorem{thm}{\protect\theoremname}
  \theoremstyle{plain}
  \newtheorem{prop}[thm]{\protect\propositionname}
  \theoremstyle{remark}
  \newtheorem*{rem*}{\protect\remarkname}

\usepackage{aaai18}  
\usepackage{times}  
\usepackage{helvet}  
\usepackage{courier}  
\usepackage{url}  
\usepackage{graphicx}  
\nocopyright

\author{Raif M. Rustamov \& James T. Klosowski\\
AT\&T Labs Research} 

\usepackage{booktabs}       
\usepackage{amsfonts}       
\usepackage{nicefrac}       
\usepackage{microtype}      

\usepackage{caption}
\renewcommand{\vec}[1]{\mathbf{#1}}
\newcommand{\node}{s}
\newcommand{\tran}{\mkern-1.5mu\mathsf{T}}

\makeatother

\usepackage{babel}
  \providecommand{\propositionname}{Proposition}
  \providecommand{\remarkname}{Remark}
\providecommand{\theoremname}{Theorem}

\begin{document}

\title{Interpretable Graph-Based Semi-Supervised Learning via Flows}
\maketitle
\begin{abstract}
In this paper, we consider the interpretability of the foundational
Laplacian-based semi-supervised learning approaches on graphs. We
introduce a novel flow-based learning framework that subsumes the
foundational approaches and additionally provides a detailed, transparent,
and easily understood expression of the learning process in terms
of graph flows. As a result, one can visualize and interactively explore
the precise subgraph along which the information from labeled nodes
flows to an unlabeled node of interest. Surprisingly, the proposed
framework avoids trading accuracy for interpretability, but in fact
leads to improved prediction accuracy, which is supported both by
theoretical considerations and empirical results. The flow-based framework
guarantees the maximum principle by construction and can handle directed
graphs in an out-of-the-box manner. 
\end{abstract}

\section{Introduction}

Classification and regression problems on networks and data clouds
can often benefit from leveraging the underlying connectivity structure.
One way of taking advantage of connectivity is provided by graph-based
semi-supervised learning approaches, whereby the labels, or values,
known on a subset of nodes, are propagated to the rest of the graph.
Laplacian-based approaches such as Harmonic Functions \cite{ZhuGahramani}
and Laplacian Regularization \cite{Belkin:2006} epitomize this class
of methods. Although a variety of improvements and extensions have
been proposed \cite{zhu05survey,Belkin:2006,DBLP:journals/jmlr/ZhouB11,partially_absorbing,icml2014c1_solomon14},
the interpretability of these learning algorithms has not received
much attention and remains limited to the analysis of the obtained
prediction weights. In order to promote accountability and trust,
it is desirable to have a more transparent representation of the prediction
process that can be visualized, interactively examined, and thoroughly
understood.

Despite the label propagation intuition behind these algorithms, devising
interpretable versions of Harmonic Functions (HF) or Laplacian Regularization
(LR) is challenging for a number of reasons. First, since these algorithms
operate on graphs in a global manner, any interactive examination
of the prediction process would require visualizing the underlying
graph, which becomes too complex for even moderately sized graphs.
Second, we do not have the luxury of trading prediction accuracy for
interpretability: HF/LR have been superseded by newer methods and
we cannot afford falling too far behind the current state of the art
in terms of the prediction accuracy. Finally, HF and LR possess useful
properties such as linearity and maximum principle that are worth
preserving.

In this paper, we introduce a novel flow-based semi-supervised learning
framework that subsumes HF and LR as special cases, and overcomes
all of these challenges. The key idea is to set up a flow optimization
problem for each unlabeled node, whereby the flow sources are the
labeled nodes, and there is a single flow destination---the unlabeled
node under consideration. The source nodes can produce as much out-flow
as needed; the destination node requires a unit in-flow. Under these
constraints, the flow optimizing a certain objective is computed,
and the amount of flow drawn from each of the labeled nodes gives
the prediction weights. The objective function contains an $\ell_{1}$-norm
like term, whose strength allows controlling the sparsity of the flow
and, as a result, its spread over the graph. When this term is dropped,
we prove that this scheme can be made equivalent to either the HF
or LR method.

This approach was chosen for several reasons. First, the language
of flows provides a detailed, transparent, and easily understood expression
of the learning process which facilitates accountability and trust.
Second, the sparsity term results in flows that concentrate on smaller
subgraphs; additionally, the flows induce directionality on these
subgraphs. Smaller size and directedness allow using more intuitive
graph layouts \cite{Gansner:2000:OGV:358668.358697}. As a result,
one can visualize and interactively explore the precise subgraph along
which the information from labeled nodes flows to an unlabeled node
of interest. Third, the sparsity term injects locality into prediction
weights, which helps to avoid flat, unstable solutions observed with
pure HF/LR in high-dimensional settings \cite{NIPS2009_3652}. Thus,
not only do we avoid trading accuracy for interpretability, but in
fact we gain in terms of the prediction accuracy. Fourth, by construction,
the flow-based framework is linear and results in solutions that obey
the maximum principle, guaranteeing that the predicted values will
stay within the range of the provided training values. Finally, directed
graphs are handled out-of-the-box and different weights for forward
and backward versions of an edge are allowed.

The main contribution of this paper is the proposed flow-based framework
(Section \ref{sec:Flow-based-Weights}). We investigate the theoretical
properties of the resulting prediction scheme (Section \ref{sec:Theoretical-Properties})
and introduce its extensions (Section \ref{sec:Extensions}). After
providing computational algorithms that effectively make use of the
shared structure in the involved flow optimization problems (Section
\ref{sec:Computation}), we present an empirical evaluation both on
synthetic and real data (Section \ref{sec:Experiments}).

\section{\label{sec:Flow-based-Weights}Flow-based Framework}

We consider the transductive classification/regression problem formulated
on a graph $G=(V,E)$ with node-set $V=\{1,2,...,n\}$ and edge-set
$E=\{1,2,...,m\}$. A function $f(\cdot)$ known on a subset of labeled
nodes $1,2,...,n_{l}$ needs to be propagated to the unlabeled nodes
$n_{l}+1,n_{l}+2,...,n$. We concentrate on approaches where the predicted
values depend linearly on the values at labeled nodes. Such linearity
immediately implies that the predicted value at an unlabeled node
$\node$ is given by 
\begin{equation}
f(\node)=\sum_{i=1}^{n_{l}}w_{i}(\node)f(i),\label{eq:InterpolationScheme}
\end{equation}
where for each $i=1,2,...,n_{l}$, the weight $w_{i}(\cdot)$ captures
the contribution of the labeled node $i$.

In this section, we assume that the underlying graph is undirected,
and that the edges $e\in E$ of the graph are decorated with dissimilarities
$d_{e}>0$, whose precise form will be discussed later. For an unlabeled
node $\node$, our goal is to compute the weights $w_{i}(\node),i=1,2,...,n_{l}$.
To this end, we set up a flow optimization problem whereby the flow
sources are the labeled nodes, and there is a single flow destination:
the node $\node$. The source nodes can produce as much out-flow as
needed; the sink node $\node$ requires a unit in-flow. Under these
constraints, the flow optimizing a certain objective function is computed,
and the amount of flow drawn from each of the source nodes $i=1,2,...,n_{l}$
gives us the desired prediction weights $w_{i}(\node)$.

With these preliminaries in mind, we next write out the flow optimization
problem formally. We first arbitrarily orient each edge of the undirected
graph, and let $A$ be $n\times m$ signed incidence matrix of the
resulting directed graph: for edge $e$, from node $i$ to node $j$,
we have $A_{ie}=+1$ and $A_{je}=-1$. This matrix will be useful
for formulating the flow conservation constraint. Note that the flow
conservation holds at unlabeled nodes, so we will partition $A$ into
two blocks. Rows $1...n_{l}$ of $A$ will constitute the $n_{l}\times m$
matrix $A_{l}$ and the remaining rows make up $A_{u}$; these sub-matrices
correspond to labeled and unlabeled nodes respectively. The right
hand side for the conservation constraints is captured by the $(n-n_{l})\times1$
column-vector $\vec{b}_{\node}$ whose entries correspond to unlabeled
nodes. All entries of $\vec{b}_{\node}$ are zero except the $\node$-th
entry which is set to $-1$, capturing the fact that there is a unit
in-flow at the sink node $\node$. For each edge $e$, let the sought
flow along that edge be $x_{e}$, and let $\vec{x}$ be the column-vector
with $e$-th entry equal to $x_{e}$. Our flow optimization problem
is formulated as:

\begin{equation}
\min_{\vec{x}\in\mathbb{R}^{m}}\frac{1}{2}\sum_{e=1}^{m}d_{e}(x_{e}^{2}+\lambda|x_{e}|)\quad\mbox{subject to}\;A_{u}\vec{x}=\vec{b}_{\node},\label{eq:UndirectedOptimization}
\end{equation}
where $\lambda\geq0$ is a trade-off parameter.

The prediction weight $w_{i}(\node)$ is given by the flow amount
drawn from the labeled node $i$. More precisely, the weight vector
is computed as $\vec{w}(\node)=A_{l}\vec{x}$ for the optimal $\vec{x}$.
Note that the optimization problem is strictly convex and, as a result,
has a unique solution. Furthermore, the optimization problem must
be solved separately for each unlabeled $\node$, i.e. for all different
vectors $\vec{b}_{\node}$, and the predicted value at node $\node$
is computed via $f(\node)=\sum_{i=1}^{n_{l}}w_{i}(\node)f(i)=(\vec{w}(\node))^{\tran}\vec{f}_{l}.$

\paragraph*{Flow Subgraphs}

Our optimization problem resembles that of elastic nets \cite{ElementsStatisticalLearning},
and the $\ell_{1}$-norm like first-order term makes the solution
sparse. For a given value of $\lambda,$ taking the optimal flow's
support---all edges that have non-zero flow values and nodes incident
to these edges---we obtain a subgraph $G_{\node}(\lambda)$ along
which the flows get propagated to the node $\node$. \textit{This
subgraph together with the flow values on the edges constitutes an
expressive summary of the learning process, and can be analyzed or
visualized for further analysis.} Potentially, further insights can
be obtained by considering $G_{\node}(\lambda)$ for different values
of $\lambda$ whereby one gets a ``regularization path'' of flow
subgraphs.

In addition, the optimal flow induces directionality on $G_{\node}(\lambda)$
as follows. While the underlying graph was oriented arbitrarily, we
get a preferred orientation on $G_{\node}(\lambda)$ by retaining
the directionality of the edges with positive flows, and flipping
the edges with negative flows; this process completely removes arbitrariness
since the edges of $G_{\node}(\lambda)$ have non-zero flow values.
In addition, this process makes the optimal flow positive on $G_{\node}(\lambda)$;
it vanishes on the rest of the underlying graph.

\paragraph*{Discussion}

To understand the main features of the proposed framework, it is instructive
to look at the tension between the quadratic and first-order terms
in the objective function. While the first-order term tries to concentrate
flows along the shortest paths, the quadratic term tries to spread
the flow broadly over the graph. This is formalized in Section \ref{sec:Theoretical-Properties}
by considering two limiting cases of the parameter $\lambda$ and
showing that our scheme provably reduces to the HF when $\lambda=0$,
and to the 1-nearest neighbor prediction when $\lambda\to\infty$.
Thanks to the former limiting case, our scheme inherits from HF the
advantage of leveraging the community structure of the underlying
graph. The latter case injects locality and keeps the spread of the
flow under control.

Limiting the spread of the flow is crucially important for interactive
exploration. The flow subgraphs $G_{\node}(\lambda)$ get smaller
with the increasing parameter $\lambda>0$. As a result, these subgraphs,
which serve as the summary of the learning process, can be more easily
visualized and interactively explored. Another helpful factor is that,
with the induced orientation, $G_{\node}(\lambda)$ is a directed
acyclic graph (DAG); see Section \ref{sec:Theoretical-Properties}
for a proof. The resulting topological ordering gives an overall direction
and hierarchical structure to the flow subgraph, rendering it more
accessible to a human companion due to the conceptual resemblance
with the commonly used flowchart diagrams.

Limiting the spread of the flow seemingly inhibits fully exploiting
the connectivity of the underlying graph. Does this lead to a less
effective method? Quite surprisingly the complete opposite is true.
HF and LR have been known to suffer from flat, unstable solutions
in high-dimensional settings \cite{NIPS2009_3652}. An insightful
perspective on this phenomenon was provided in \cite{getting_lost,DBLP:journals/jmlr/LuxburgRH14}
which showed that random walks of the type used in Laplacian-based
methods spread too thinly over the graph and ``get lost in space''
and, as a result, carry minimal useful information. Therefore, limiting
the spread within our framework can be seen as an antidote to this
problem. Indeed, as empirically confirmed in Section \ref{sec:Experiments},
in contrast to HF/LR which suffer from almost uniform weights, flow-based
weights with $\lambda>0$ concentrate on a sparse subset of labeled
nodes and help to avoid the flat, unstable solutions.

\section{\label{sec:Theoretical-Properties}Theoretical Properties}

In this section we prove a number of properties of the proposed framework.
We first study its behavior at the limiting cases of $\lambda=0$
and $\lambda\to\infty$. Next, we prove that for all values of the
parameter $\lambda$, the subgraphs $G_{\node}(\lambda)$ supporting
the flow are acyclic and the maximum principle holds.

\paragraph*{Limiting Behavior}

Introduce the column-vector $\vec{d}$ consisting of dissimilarities
$d_{e}$, and the diagonal matrix $D=\mathrm{diag}(\vec{d})$. It
is easy to see that the matrix $L=AD^{-1}A^{\tran}$ is the un-normalized
Laplacian of the underlying graph with edge weights (similarities)
given by $1/d_{e}$. As usual, the edge weights enter the Laplacian
with a negative sign, i.e. for each edge $e=(i,j)$ we have $L_{ij}=-1/d_{ij}$. 
\begin{prop}
\label{prop:equivalence}When $\lambda=0$, the flow-based prediction
scheme is equivalent to the Harmonic Functions approach with the Laplacian
$L=AD^{-1}A^{\tran}$.\end{prop}
\begin{proof}
The HF method uses the Laplacian matrix $L$ and constructs predictions
by optimizing $\min_{\vec{f}\in\mathbb{R}^{n}}\sum_{(i,j)\in E}\,-L_{ij}(f_{i}-f_{j})^{2}$
subject to reproducing the values at labeled nodes, $f_{i}=f(i)$
for $i\in\{1,2,...,n_{l}\}$. By considering the partitions of the
Laplacian along labeled and unlabeled nodes, we can write the solution
of the HF method as $\vec{f}_{u}=-L_{uu}^{-1}L_{ul}\vec{f}_{l}$,
compare to Eq. (5) in \cite{ZhuGahramani}.

When $\lambda=0$, the flow optimization problem Eq. (\ref{eq:UndirectedOptimization})
is quadratic with linear constraints, and so can be carried out in
a closed form (see e.g. Section 4.2.5 of \cite{DBLP:journals/ftml/BoydPCPE11}).
The optimal flow vector is given by $\vec{x}=D^{-1}A_{u}^{\tran}(A_{u}D^{-1}A_{u}^{\tran})^{-1}\vec{b}_{\node}$,
and the weights are computed as $\vec{w}(\node)=A_{l}\vec{x}=A_{l}D^{-1}A_{u}^{\tran}(A_{u}D^{-1}A_{u}^{\tran})^{-1}\vec{b}_{\node}$.
When we plug these weights into the prediction formula (\ref{eq:InterpolationScheme}),
we obtain the predicted value at $\node$ as $f(\node)=(\vec{w}(\node))^{\tran}\vec{f}_{l}=\vec{b}_{\node}^{\tran}(A_{u}D^{-1}A_{u}^{\tran})^{-1}A_{u}D^{-1}A_{l}^{\tran}\vec{f}_{l}.$
Since $\vec{b}_{\node}$ is zero except at the position corresponding
to $\node$ where it is $-1$ (this is the reason for the negative
sign below), we can put together the formulas for all separate $\node$
into a single expression 
\[
\vec{f}_{u}=-(A_{u}D^{-1}A_{u}^{\tran})^{-1}A_{u}D^{-1}A_{l}^{\tran}\vec{f}_{l}.
\]
By using the Laplacian $L=AD^{-1}A^{\tran}$ and considering its partitions
along labeled and unlabeled nodes, we can re-write the formula as
$\vec{f}_{u}=-L_{uu}^{-1}L_{ul}\vec{f}_{l}$, giving the same solution
as the HF method.\end{proof}
\begin{rem*}
The converse is true as well: for any Laplacian $L$ built using non-negative
edge weights (recall that off-diagonal entries of such Laplacians
are non-positive), the same predictions as HF can be obtained via
our approach at $\lambda=0$ with appropriate costs $d_{e}$. When
the Laplacian matrix $L$ is symmetric, this easily follows from Proposition
\ref{prop:equivalence} by simply setting $d_{ij}=-1/L_{ij}$ for
all $(i,j)\in E$. However, when $L$ is not symmetric (e.g. after
Markov normalization), we can still match the HF prediction results
by setting $d_{ij}=-2/(L_{ij}+L_{ji})$ for all $(i,j)\in E$. The
main observation is that while $L$ may be asymmetric, it can always
be symmetrized without inducing any change in the optimization objective
of the HF method. Indeed the objective $\sum_{(i,j)\in E}\,-L_{ij}(f_{i}-f_{j})^{2}$
is invariant to setting $L_{ij}\leftarrow(L_{ij}+L_{ji})/2$. Now
using Proposition \ref{prop:equivalence}, we see that by letting
$d_{ij}=-2/(L_{ij}+L_{ji})$, our framework with $\lambda=0$ reproduces
the same predictions as the HF method.\end{rem*}
\begin{prop}
When one formally sets $\lambda=\infty$, the flow based prediction
scheme is equivalent to the 1-nearest neighbor prediction.\end{prop}
\begin{proof}
In this setting, the first-order term dominates the cost, and the
flow converges to a unit flow along the shortest path (with respect
to costs $d_{e}$) from the closest labeled node, say $i^{*}(p)$,
to node $p$. We can then easily see that the resulting prediction
weights are all zero, except $w_{i^{*}(p)}(p)=1$. Thus, the prediction
scheme (\ref{eq:InterpolationScheme}) becomes equivalent to the 1-nearest
neighbor prediction. 
\end{proof}

\paragraph*{Acyclicity of Flow Subgraphs}

Recall that the optimal solution induces a preferred orientation on
subgraph $G_{\node}(\lambda)$; this orientation renders the optimal
flow positive on $G_{\node}(\lambda)$ and zero on the rest of the
underlying graph. 
\begin{prop}
\label{prop:DAG}For all $\lambda\geq0$, the oriented flow subgraph
$G_{\node}(\lambda)$ is acyclic.\end{prop}
\begin{proof}
Suppose that $G_{\node}(\lambda)$ has a cycle of the form $k_{1}\rightarrow...\rightarrow k_{r}\to k_{1}$.
The optimal flow $\vec{x}$ is positive along all of the edges of
$G_{\node}(\lambda)$, including the edges in this cycle; let $x_{0}>0$
be the smallest of the flow values on the cycle. Let $\vec{x}_{\mathrm{sub}}$
be the flow vector corresponding to the flow along this cycle with
constant flow value of $x_{0}$. Note that $\vec{x}-\vec{x}_{\mathrm{sub}}$
is a feasible flow, and it has strictly lower cost than $\vec{x}$
because the optimization objective is diminished by a decrease in
the components of $\vec{x}$. This contradiction proves acyclicity
of $G_{\node}(\lambda)$. 
\end{proof}

\paragraph*{Maximum Principle}

It is well-known that harmonic functions satisfy the maximum principle.
Here we provide its derivation in terms of the flows, showing that
the maximum principle holds in our framework for all settings of the
parameter $\lambda$. 
\begin{prop}
\label{prop:convexity}For all $\lambda\geq0$ and for all $\node$,
we have $\forall i,w_{i}(\node)\geq0$ and $\sum_{i=1}^{n_{l}}w_{i}(\node)=1$. \end{prop}
\begin{proof}
Non-negativity of weights holds because a flow having a negative out-flow
at a labeled node can be made less costly by removing the corresponding
sub-flow. The formal proof is similar to the proof or Proposition
\ref{prop:DAG}. First, make the flow non-negative by re-orienting
the underlying graph (which was oriented arbitrarily). Now suppose
that the optimal solution $\vec{x}$ results in $w_{i}(\node)<0$
for some $i$, meaning that the labeled node $i$ has an in-flow of
magnitude $|w_{i}(\node)|$. Since all of the flow originates from
labeled nodes, there must exist $j\in\{1,...,n_{l}\}/\{i\}$ and a
flow-path $j\rightarrow k_{1}\rightarrow...\rightarrow k_{r}\rightarrow i$
with positive flow values along the path; let $x_{0}>0$ be the smallest
of these flow values. Let $\vec{x}_{\mathrm{sub}}$ be the flow vector
corresponding to the flow along this path with constant flow value
of $x_{0}$. Note that $\vec{x}-\vec{x}_{\mathrm{sub}}$ is a feasible
flow, and it has strictly lower cost than $\vec{x}$ because the optimization
objective is diminished by a decrease in the components of $\vec{x}$.
This contradiction proves the positivity of weights.

The weights, or equivalently the out-flows from the labeled nodes,
add up to $1$ because the destination node absorbs a total of a unit
flow. More formally, note that by definition, $\vec{1}^{\tran}A=\vec{0}$.
Splitting this along labeled and unlabeled nodes we get $\vec{1}^{\tran}A=\vec{1}_{l}^{\tran}A_{l}+\vec{1}_{u}^{\tran}A_{u}=\vec{0}$.
Right-multiplying by $\vec{x}$ gives $\vec{1}_{l}^{\tran}A_{l}\vec{x}+\vec{1}_{u}^{\tran}A_{u}\vec{x}=\vec{1}_{l}^{\tran}\vec{w}(\node)+\vec{1}_{u}^{\tran}\vec{b}_{\node}=\vec{1}_{l}^{\tran}\vec{w}(\node)-1=0$,
or $\sum_{i=1}^{n_{l}}w_{i}(\node)=1$ as desired. 
\end{proof}
This proposition together with Eq. (\ref{eq:InterpolationScheme})
guarantees that the function $f(\cdot)$ obeys the maximum principle.

\section{\label{sec:Extensions}Extensions}

\paragraph*{LR and Noisy Labels}

Laplacian Regularization has the advantage of allowing to train with
noisy labels. Here, we discuss modifications needed to reproduce LR
via our framework in the limit $\lambda=0$. This strategy allows
incorporating noisy training labels into the flow framework for all
$\lambda\geq0$.

Consider the LR objective $\mu^{-1}\sum_{i=1}^{n_{l}}(f_{i}\text{\textendash}f(i))^{2}+\sum_{(i,j)\in E}\,-L_{ij}(f_{i}-f_{j})^{2}$,
where $f(i)$ are the provided noisy labels/values and $\mu$ is the
strength of Laplacian regularization. In this formulation, we need
to learn $f_{i}$ for both labeled and unlabeled nodes. The main observation
is that the soft labeling terms $\mu^{-1}(f_{i}\lyxmathsym{\textendash}f(i))^{2}$
can be absorbed into the Laplacian by modifying the data graph. For
each labeled node $i\in\{1,2,...,n_{l}\}$, introduce a new anchor
node $i^{a}$ with a single edge connecting $i^{a}$ to node $i$
with the weight of $\mu^{-1}/2$ (halving offsets doubling in the
regularizer sum). Consider the HF learning on the modified graph,
where the labeled nodes are the anchor nodes only---i.e. optimize
the HF objective on the modified graph with hard constraints $f_{i^{a}}=f(i)$;
clearly, this is equivalent to LR. Given the relationship between
HF and the flow formulation, this modification results in a flow formulation
of LR at the limiting case of $\lambda=0$.

\paragraph*{Directed Graphs}

Being based on flows, our framework can treat directed graphs very
naturally. Indeed, since the flow can only go along the edge direction,
we have a new constraint $x_{e}\geq0$, giving the optimization problem:
\begin{equation}
\min_{\vec{x}\in\mathbb{R}_{+}^{m}}\frac{1}{2}\sum_{e=1}^{m}d_{e}(x_{e}^{2}+\lambda x_{e})\quad\mbox{subject to}\;A_{u}\vec{x}=\vec{b}_{\node},\label{eq:DirectedOptimizationProblem}
\end{equation}
Even when $\lambda=0$, this scheme is novel---due to the non-negativity
constraint the equivalence to HF no longer holds. The naturalness
of our formulation is remarkable when compared to the existing approaches
to directed graphs that use different graph normalizations \cite{DBLP:conf/nips/ZhouSH04,DBLP:conf/icml/ZhouHS05},
co-linkage analysis \cite{directed_colincage}, or asymmetric dissimilarity
measures \cite{subramanya11}.

This formulation can also handle graphs with asymmetric weights. Namely,
one may have different costs for going forward and backward: the cost
of the edge $i\rightarrow j$ can differ from that of $j\rightarrow i$.
If needed, such cost differentiation can be applied to undirected
graphs by doubling each edge into forward and backward versions. In
addition, it may be useful to add opposing edges (with higher costs)
into directed graphs to make sure that the flow problem is feasible
even if the underlying graph is not strongly connected.

\section{\label{sec:Computation}Computation}

The flow optimization problems discussed in Section \ref{sec:Extensions}
can be solved by the existing general convex or convex quadratic solvers.
However, general purpose solvers cannot use the shared structure of
the problem---namely that everything except the vector $\vec{b}_{\node}$
is fixed. Here, we propose solvers based on the Alternating Direction
Method of Multipliers (ADMM) \cite{DBLP:journals/ftml/BoydPCPE11}
that allow caching the Cholesky factorization of a relevant matrix
and reusing it in all of the iterations for all unlabeled nodes. We
concentrate on the undirected version because of the space limitations.
\begin{figure*}
\begin{centering}
\includegraphics[width=1\textwidth]{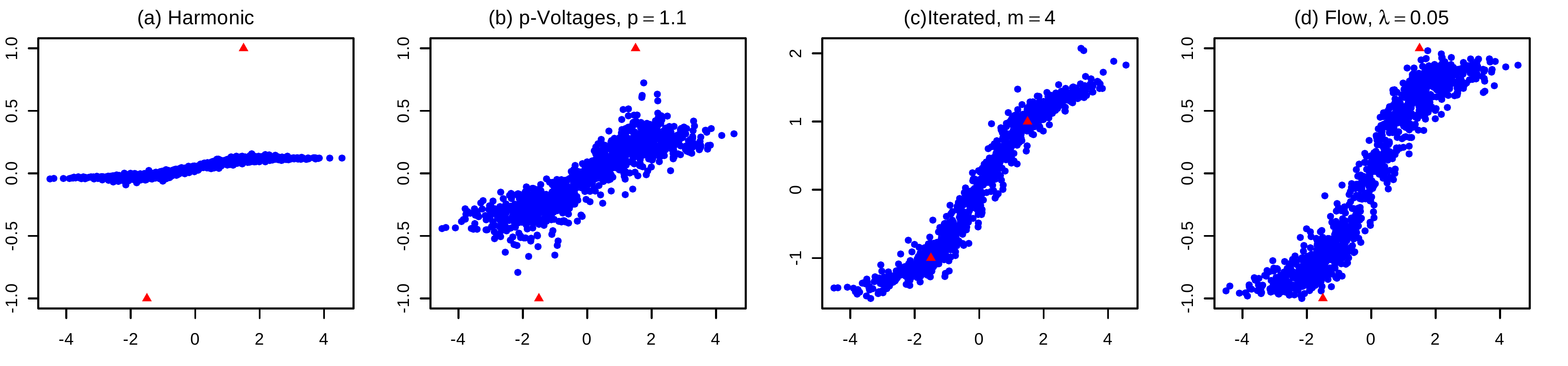} 
\par\end{centering}

\caption{\label{fig:Gaussian-mixture}$f(\cdot)$ for a mixture of two Gaussians
in $\mathbb{R}^{10}$ for a number of methods.}
\end{figure*}

In the ADMM form, the problem (\ref{eq:UndirectedOptimization}) can
be written as 
\begin{align*}
\min_{\vec{x},\vec{z}\in\mathbb{R}^{m}}\;\, & g(\vec{x})+\frac{1}{2}\sum_{e\in E}d_{e}x_{e}^{2}+\lambda\sum_{e\in E}d_{e}|z_{e}|\\
\mbox{subj to}\;\, & \vec{x}-\vec{z}=\vec{0},
\end{align*}
where $g(x)$ is the $0/\infty$ indicator function of the set $\{\vec{x}|A_{u}\vec{x}=\vec{b}_{\node}\}$.
Stacking the dissimilarities $d_{e}$ into a column-vector \textbf{d},
and defining the diagonal matrix $D=\mathrm{diag}(\vec{d})$, the
ADMM algorithm then consists of iterations: 
\begin{align*}
\vec{x}^{k+1}:= & \arg\min_{\vec{x}}\:g(\vec{x})+\frac{1}{2}\vec{x}^{\tran}D\vec{x}+\frac{\rho}{2}\Vert\vec{x}-\vec{z}^{k}+\vec{u}^{k}\Vert_{2}^{2}\\
\vec{z}^{k+1}:= & S_{\vec{d}\lambda/\rho}(\vec{x}^{k+1}+\vec{u}^{k})\\
\vec{u}^{k+1}:= & \vec{u}^{k}+\vec{x}^{k+1}-\vec{z}^{k+1}
\end{align*}
Here, $S_{\vec{a}}(\vec{b})=(\vec{b}-\vec{a})_{+}-(-\vec{b}-\vec{a})_{+}$
is the component-wise soft-thresholding function.

The $\vec{x}$-iteration can be computed in a closed form as the solution
of an equality constrained quadratic program, cf. Section 4.2.5 of
\cite{DBLP:journals/ftml/BoydPCPE11}. Letting $M=\mathrm{diag}(1/(\rho+\vec{d}))$,
we first solve the linear system $A_{u}MA_{u}^{\tran}\vec{y}=2\rho A_{u}M(\vec{z}^{k}-\vec{u}^{k})-\vec{b}_{\node}$
for $\vec{y}$, and then let $\vec{x}^{k+1}:=2\rho M(\vec{z}^{k}-\vec{u}^{k})-MA_{u}^{\mathsf{T}}\vec{y}$.
The most expensive step is the involved linear solve. Fortunately,
the matrix $A_{u}MA_{u}^{\tran}$ is both sparse (same fill as the
graph Laplacian) and positive-definite. Thus, we compute its Cholesky
factorization and use it throughout all iterations and for all unlabeled
nodes, since $A_{u}MA_{u}^{\tran}$ contains only fixed quantities. 

We initialize the iterations by setting $\vec{z}^{0}=\vec{u}^{0}=\vec{0}$,
and declare convergence when both $\max(\vec{x}^{k+1}-\vec{z}^{k+1})$
and $\max(\vec{z}^{k+1}-\vec{z}^{k})$ fall below a pre-set threshold.
As explained in Section 3.3 of \cite{DBLP:journals/ftml/BoydPCPE11},
these quantities are related to primal and dual feasibilities. Upon
convergence, the prediction weights are computed using the formula
$\vec{w}(\node)=A_{l}\vec{z}$; thanks to soft-thresholding, using
$\vec{z}$ instead of $\vec{x}$ avoids having a multitude of negligible
non-zero entries.

\section{\label{sec:Experiments}Experiments}

In this section, we validate the proposed framework and then showcase
its interpretability aspect.

\subsection{Validation}

We experimentally validate a number of claims made about the flow
based approach: 1) limiting the spread of the flow via sparsity term
results in improved behavior in high-dimensional settings; 2) this
improved behavior holds for real-world data sets and leads to increased
predictive accuracy over HF; 3) our approach does not fall far behind
the state-of-the-art in terms of accuracy. We also exemplify the directed/asymmetric
formulation on a synthetic example.

We compare the flow framework to the original HF formulation and also
to two improvements of HF designed to overcome the problems encountered
in high-dimensional settings: \textit{$p$-Voltages}---the method
advocated in Alamgir and von Luxburg \cite{DBLP:conf/nips/AlamgirL11}
(they call it $q$-Laplacian regularization) and further studied in
\cite{pvoltages,DBLP:conf/colt/Alaoui16}; \textit{Iterated Laplacians}---a
state-of-the-art method proposed in \cite{DBLP:journals/jmlr/ZhouB11}.

For data clouds we construct the underlying graph as a weighted $20$-nearest
neighbor graph. The edge weights are computed using Gaussian RBF with
$\sigma$ set as one-third of the mean distance between a point and
its tenth nearest neighbor \cite{ChaSchZie06}. The normalized graph
Laplacian $L$ is used for HF and Iterated Laplacian methods. The
edge costs for the flow approach are set by $d_{ij}=-2/(L_{ij}+L_{ji})$
as described in Section \ref{sec:Theoretical-Properties}. The weights
for $p$-Voltages are computed as $d_{ij}^{-1/(p-1)}$, see \cite{DBLP:conf/nips/AlamgirL11,pvoltages}.

\paragraph*{Mixture of Two Gaussians}

Consider a mixture of two Gaussians in $10$-dimensional Euclidean
space, constructed as follows. We sample $500$ points from each of
the two Gaussians with $\vec{\mu}_{1}=\vec{\mu}_{2}=\vec{0}$ and
$\Sigma_{1}=\Sigma_{2}=I$. The points from each Gaussian are respectively
shifted by $-1.5$ and $+1.5$ units along the first dimension. A
single labeled point is used for each Gaussian; the labels are $-1$
for the left Gaussian and $+1$ for the right one. According to \cite{DBLP:conf/nips/AlamgirL11,DBLP:conf/colt/Alaoui16},
an appropriate value of $p$ for $p$-Voltages is $(10+1)/10=1.1$.

In Figure \ref{fig:Gaussian-mixture}, we plot the estimated functions
$f$ for various methods. The sampled points are projected along the
first dimension, giving the $x$-axis in these plots; the $y$-axis
shows the estimator $f$. Figure \ref{fig:Gaussian-mixture} (a) confirms
the flatness of the HF solution as noted in \cite{NIPS2009_3652}.
In classification problems flatness is undesirable as the solution
can be easily shifted by a few random labeled samples to favor one
or the other class. For regression problems, the estimator $f$ is
unsuitable as it fails to be smooth, which can be seen by the discrepancy
between the values at labeled nodes (shown as red triangles) and their
surrounding.

The $p$-Voltages solution, Figure \ref{fig:Gaussian-mixture} (b),
avoids the instability issue, but is not completely satisfactory in
terms of smoothness. The iterated Laplacian and flow methods, Figure
\ref{fig:Gaussian-mixture} (c,d), suffer neither from flatness nor
instability. Note that in contrast to iterated Laplacian method, the
flow formulation satisfies the maximum principle, and $f$ stays in
the range $[-1,1]$ established by the provided labels.

\begin{table}
\centering \begin{scriptsize} %
\begin{tabular}{lllll}
\toprule 
Dataset  & Harmonic  & p-Voltages  & Iterated  & Flow \tabularnewline
\midrule 
MNIST 3vs8  & $8.5\pm2.4$  & $7.8\pm1.8$  & $6.1\pm2.3$  & $6.2\pm1.6$ \tabularnewline
MNIST 4vs9  & $22.3\pm8.2$  & $13.3\pm2.8$  & $8.5\pm2.0$  & $9.6\pm2.7$ \tabularnewline
AUT-AVN  & $27.9\pm13.7$  & $19.0\pm5.1$  & $11.5\pm2.0$  & $14.6\pm2.7$ \tabularnewline
CCAT  & $33.3\pm8.5$  & $25.4\pm3.5$  & $21.5\pm3.4$  & $22.3\pm3.2$ \tabularnewline
GCAT  & $20.0\pm11.6$  & $13.6\pm3.7$  & $9.2\pm2.6$  & $10.0\pm1.7$ \tabularnewline
PCMAC  & $30.6\pm11.2$  & $21.4\pm3.2$  & $14.1\pm1.6$  & $18.2\pm3.3$ \tabularnewline
REAL-SIM  & $31.5\pm12.1$  & $22.4\pm3.6$  & $15.3\pm2.8$  & $17.6\pm3.3$ \tabularnewline
\bottomrule
\end{tabular}\end{scriptsize} \caption{\label{tab:Misclassification-percentage}Misclassification rates (\%)
and standard deviations.}
\end{table}

\begin{table}
\centering \begin{scriptsize} \tabcolsep=0.1cm %
\begin{tabular}{llllll}
\toprule 
Dataset  & Harmonic  & $\lambda=0.025$  & $\lambda=0.05$  & $\lambda=0.1$  & $\lambda=0.2$ \tabularnewline
\midrule 
MNIST 3vs8  & $8.5\pm2.4$  & $5.8\pm1.6$  & $5.9\pm1.5$  & $6.2\pm1.4$  & $6.7\pm1.4$ \tabularnewline
MNIST 4vs9  & $22.3\pm8.2$  & $9.4\pm2.8$  & $9.3\pm2.7$  & $9.7\pm2.7$  & $10.3\pm2.7$ \tabularnewline
AUT-AVN  & $27.9\pm13.7$  & $14.0\pm2.6$  & $14.8\pm2.4$  & $16.3\pm2.3$  & $18.4\pm2.4$ \tabularnewline
CCAT  & $33.3\pm8.5$  & $21.9\pm3.0$  & $22.5\pm2.8$  & $23.2\pm2.5$  & $24.5\pm2.5$ \tabularnewline
GCAT  & $20.0\pm11.6$  & $9.9\pm1.6$  & $10.8\pm1.6$  & $12.4\pm1.8$  & $14.1\pm1.7$ \tabularnewline
PCMAC  & $30.6\pm11.2$  & $17.2\pm2.9$  & $18.0\pm2.5$  & $19.5\pm2.2$  & $21.0\pm2.2$ \tabularnewline
REAL-SIM  & $31.5\pm12.1$  & $17.1\pm3.4$  & $17.7\pm3.0$  & $19.1\pm2.7$  & $21.1\pm2.4$ \tabularnewline
\bottomrule
\end{tabular}\end{scriptsize} \caption{\label{tab:Misclassification-percentage-fixed}Misclassification rates
(\%) and standard deviations.}
\end{table}

\paragraph{Benchmark Data Sets}

Next we test the proposed method on high-dimensional image and text
datasets used in \cite{DBLP:journals/jmlr/ZhouB11}, including MNIST
3vs8, MNIST 4vs9, aut-avn, ccat, gcat, pcmac, and real-sim. For each
of these, we use a balanced subset of $1000$ samples. In each run
we use $n_{l}=50$ labeled samples; in addition, we withheld $50$
samples for validation. The misclassification rate is computed for
the remaining 900 samples. The results are averaged over $20$ runs.
For $p$-Voltages, the value of $p$ for each run is chosen from $\{1.0625,1.125,1.25,1.5,2\}$
using the best performing value on the validation samples; note that
$p=2$ corresponds to HF. For iterated Laplacian method, $m$ is chosen
from $\{1,2,4,8,16\}$ with $m=1$ being equivalent to HF. For the
flow approach, the value of $\lambda$ is chosen from $\{0,0.025,0.05,0.1,0.2\}$,
where $\lambda=0$ is equivalent to HF.

The results summarized in Table \ref{tab:Misclassification-percentage}
show that the flow approach outperforms HF and $p$-Voltages, even
with the added benefit of interpretability (which the other methods
do not have). The flow approach does have slightly lower accuracy
than the state-of-the-art iterated Laplacian method, but we believe
that this difference is a tolerable tradeoff for applications where
interpretability is required.

Next, we compare HF and the flow approach but this time instead of
selecting $\lambda$ by validation we use fixed values. The results
presented in Table \ref{tab:Misclassification-percentage-fixed} demonstrate
that the flow approach for each of the considered values consistently
outperforms the base case of HF (i.e. $\lambda=0$). Note that larger
values of $\lambda$ lead to increased misclassification rate, but
the change is not drastic. This is important for interactive exploration
because larger $\lambda$ values produce smaller flow subgraphs $G_{\node}(\lambda)$
that are easier for the human companion to grasp.

\begin{figure}
\begin{centering}
\includegraphics[width=0.91\columnwidth]{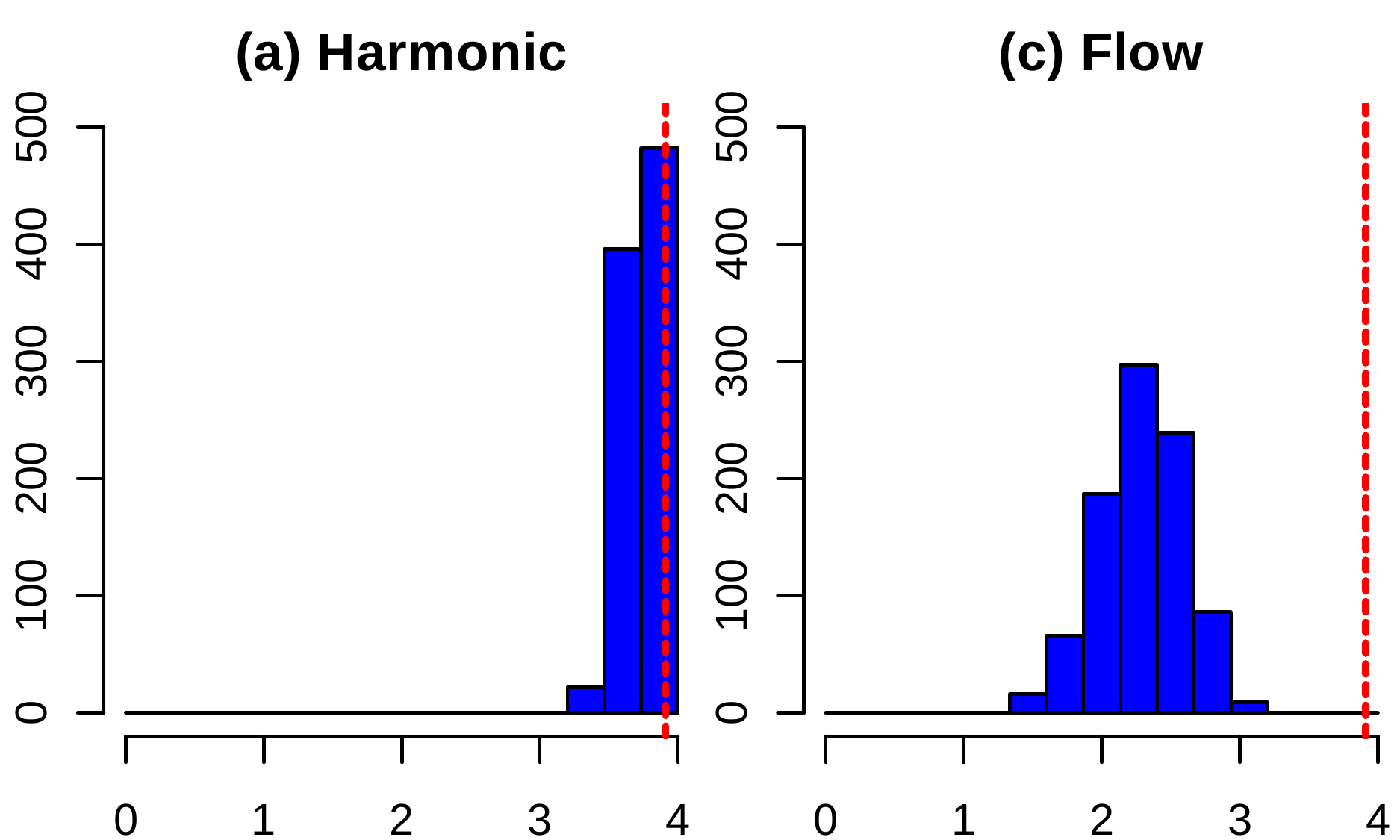} 
\par\end{centering}

\caption{\label{fig:Weight-statistics-MNIST-3vs8}Weight statistics for MNIST
3vs8.}
\end{figure}

Finally, we show that the improvement over HF observed in these experiments
relates to the flatness issue and, therefore, is brought about by
limiting the spread of the flow. By analyzing the prediction weights
in Eq. (\ref{eq:InterpolationScheme}), we can compare HF to the flow
based approach in terms of the flatness of solutions. To focus our
discussion, let us concentrate on a single run of MNIST 3vs8 classification
problem. Since the weights for the flow based and HF approaches lend
themselves to a probabilistic interpretation (by Proposition \ref{prop:convexity},
they are non-negative and add up to $1$), we can look at the entropies
of the weights, $H(\node)=-\sum_{i}w_{i}(\node)\log w_{i}(\node)$
for every unlabeled node $\node=n_{l}+1,...,n$. For a given unlabeled
$s$, the maximum entropy is achieved for uniform weights: $\forall i,w_{i}(s)=1/n_{l}$.
Figure \ref{fig:Weight-statistics-MNIST-3vs8} shows the histograms
of entropies, together with the red vertical line corresponding to
the maximum possible entropy. In contrast to the flow based approach,
the entropies from HF method are clustered closely to the red line,
demonstrating that there is little variation in HF weights, which
is a manifestation of the flatness issue.

\paragraph*{Directed/Asymmetric Graphs}

In this synthetic example, we demonstrate the change in the estimated
$f(\cdot)$ induced by the use of asymmetric edge costs in the directed
formulation given by Eq. (\ref{eq:DirectedOptimizationProblem}).
The underlying graph represents the road network of Minnesota, with
edges showing the major roads and vertices being their intersections.
As in the previous example, we pick two nodes (shown as triangles))
and label them with $\pm1$ and depict the resulting estimator $f(\cdot)$
for the remaining unlabeled nodes using the shown color-coding. In
Figure \ref{fig:EffectOfDirectionality} (a), the graph is undirected,
i.e. the edge costs are symmetric. In Figure \ref{fig:EffectOfDirectionality}
(b), we consider the directed graph obtained by doubling each edge
into forward and backward versions. This time the costs for edges
going from south to north ($\mbox{lat}_{i}<\mbox{lat}_{j}$) are multiplied
by four. The latter setting makes flows originating at the top labeled
node cheaper to travel towards south, and therefore shifts the ``decision
boundary'' to the south.

\begin{figure}
\begin{centering}
\includegraphics[width=0.85\columnwidth]{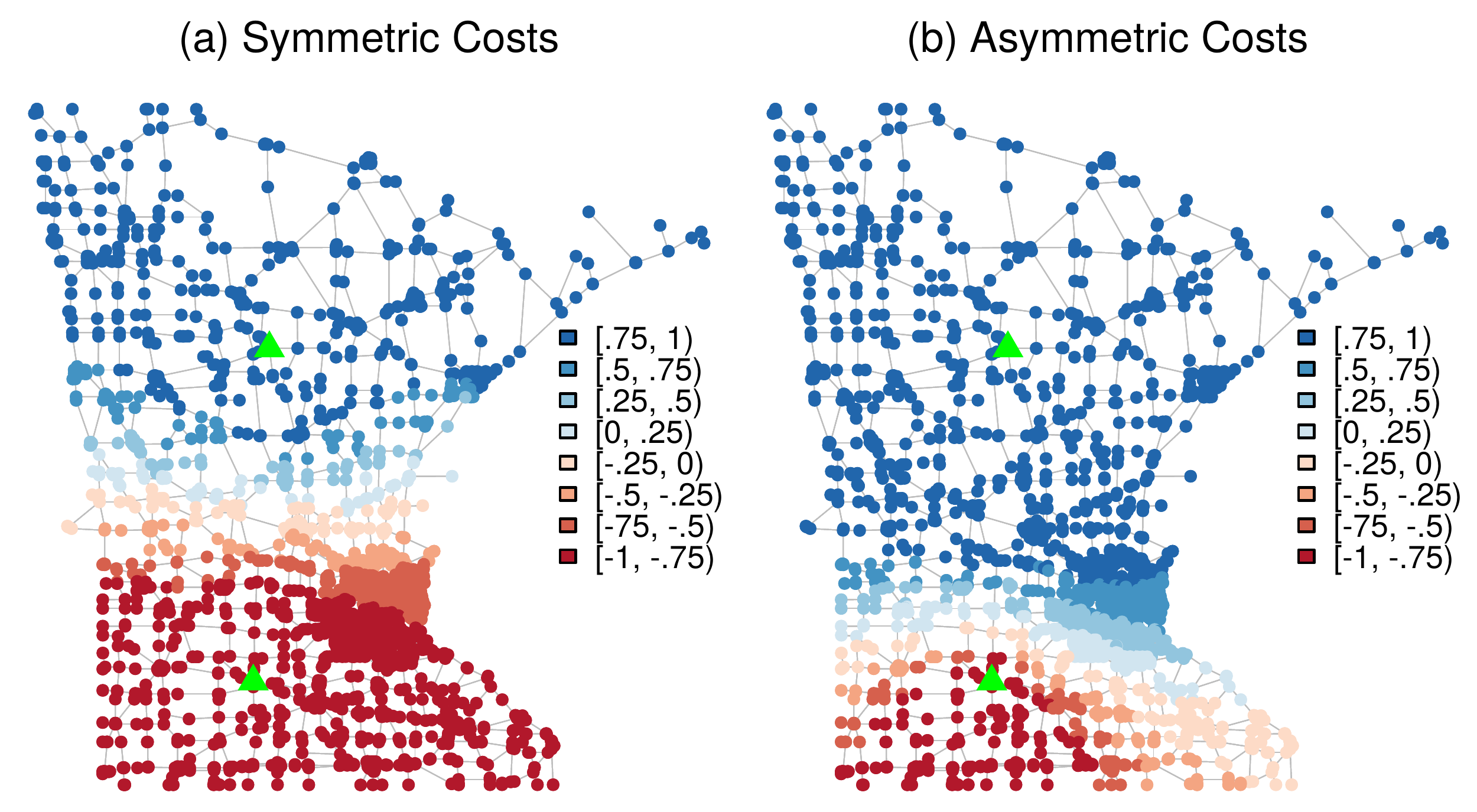} 
\par\end{centering}

\caption{\label{fig:EffectOfDirectionality}$f(\cdot)$ for symmetric and asymmetric
edge costs.}
\end{figure}

\subsection{Interpretability}

As stated previously, one of the core benefits of the proposed framework
is how it provides a detailed, transparent, and easily understood
expression of the learning process. To exemplify this aspect of the
proposed framework, we focus our discussion on a fixed data-point
$\node$ from MNIST 3vs8 and study the corresponding flow subgraphs
$G_{\node}(\lambda)$ for $\lambda=0.2,0.1$ and $0.05$; see Figure
\ref{fig:Flow-subgraphs} (also see supplementary examples). The digit
image for unlabeled node $\node$ is outlined in red; the labeled
node images are outlined in blue. The flows along the edges are in
percents, and have been rounded up to avoid the clutter.

As expected, the size of the subgraph depends on the parameter $\lambda$.
Recall that for each $\lambda>0$, the flow optimization automatically
selects a subset of edges carrying non-zero flow---this is akin to
variable selection in sparse regression \cite{ElementsStatisticalLearning}.
The figure confirms that the larger the strength of the sparsity penalty,
the fewer edges get selected, and the smaller is the resulting subgraph.
The amount of reduction in size is substantial: the underlying graph
for this experiment has 1K nodes and \textasciitilde{}15K edges.

As claimed, the acyclicity of the flow subgraphs $G_{\node}(\lambda)$
leads to more intuitive layouts. We used the ``dot'' filter from
Graphviz \cite{graphviz_dot,Gansner:2000:OGV:358668.358697} which
is targeted towards visualizing hierarchical directed graphs. Since
our flow subgraphs are DAGs, the ``dot'' filter gave satisfactory
layouts without any manual tweaking. Indeed, all of the visualizations
depicted in Figure \ref{fig:Flow-subgraphs} provide an overall sense
of directionality, here from top to bottom: due to acyclicity there
are no edges going ``backward''. This makes it easy to trace the
flow from sources, the labeled nodes, to the sink, the unlabeled node
$\node$.

Of course, the visualizations in Figure \ref{fig:Flow-subgraphs}
benefit from the fact that the data points are digits, which allows
including their images in lieu of abstract graph nodes. The same approach
could be used for general image data sets as well. For other data
types such as text, the nodes could depict some visual summary of
the data point, and then provide a more detailed summary upon a mouse
hover. 
\begin{figure}
\begin{centering}
\includegraphics[width=0.85\columnwidth]{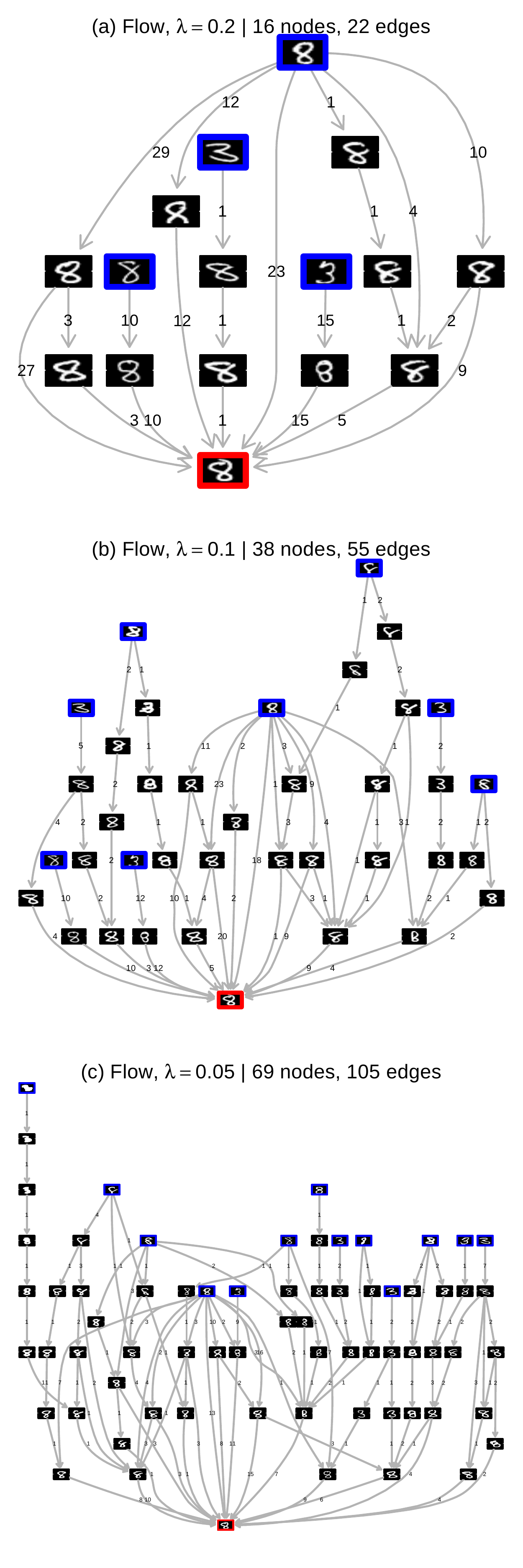} 
\par\end{centering}

\caption{\label{fig:Flow-subgraphs}Flow subgraphs $G_{\node}(\lambda)$ showing
the information propagation from the labeled nodes (outlined in blue)
to the unlabeled node $\node$ (outlined in red). The flow values
on the edges are in percents, and may fail to add up due to rounding.}
\end{figure}

\section{\label{sec:Discussion}Discussion and Future Work}

We have presented a novel framework for graph-based semi-supervised
learning that provides a transparent and easily understood expression
of the learning process via the language of flows. This facilitates
accountability and trust and makes a step towards the goal of improving
human-computer collaboration.

Our work is inspired by \cite{DBLP:conf/nips/AlamgirL11} which pioneered
the use of the flow interpretation of the standard resistance distances
(c.f. \cite{Bollobas1998}) to introduce a novel class of resistance
distances that avoid the pitfalls of the standard resistance distance
by concentrating the flow on fewer paths. They also proved an interesting
phase transition behavior that led to specific suggestions for Laplacian-based
semi-supervised learning. However, their proposal for semi-supervised
learning is not based on flows, but rather on an appropriate setting
of the parameter in the $p$-Voltages method (which they call $q$-Laplacian
regularization); it was further studied in \cite{pvoltages,DBLP:conf/colt/Alaoui16}.
Although it is already apparent from the experimental results that
our proposal is distinct from $p$-Voltages, we stress that there
is a fundamental theoretical difference: $p$-Voltages method is non-linear
and so cannot be expressed via Eq. (\ref{eq:InterpolationScheme}),
except when $p=2$. At a deeper level, these methods regularize the
estimated function via its value discrepancies at adjacent nodes,
and, therefore, do not directly provide a detailed understanding of
how the values propagate along the graph.

One immediate direction for future work is to obtain a flow formulation
for the iterated Laplacian method which gave best accuracy on the
benchmark data sets. This may seem straightforward to do as iterated
Laplacian method basically replaces the Laplacian operator $L$ in
the regularizer by its power $L^{m}$. However, the matrix $L^{m}$
contains both positive and negative off-diagonal entries, and so,
the corresponding edge costs are no longer positive, which renders
the flow interpretation problematic. An indirect manifestation of
this issue is the lack of a maximum principle for the iterated Laplacian
method. Another complicating factor is that the stencils of $L^{m}$
grow very fast with $m$, e.g. $L^{4}$ is nearly a full matrix in
the benchmark examples.

Another interesting future work direction is to explore approaches
for generating multiple levels of explanations from the flow graph.
For example, main paths of the information flow could be identified
via graph summarization techniques and presented as an initial coarse
summary, either in pictorial or textual form. As the user drills down
on particular pathways, more detailed views could be generated.

Additional avenues for future study include algorithmic improvements
and applications in other areas. The locality of the resulting weights
can be used to devise faster algorithms that operate directly on an
appropriate neighborhood of the given unlabeled node. Studying the
sparsity parameter $\lambda$ for phase transition behavior can provide
guidance for its optimal choice. In another direction, since the flow-based
weights are non-negative and add up to 1, they are suitable for semi-supervised
learning of probability distributions \cite{icml2014c1_solomon14}.
Finally, the flow-based weights can be useful in different areas,
such as descriptor based graph matching or shape correspondence in
computer vision and graphics.

\paragraph*{Acknowledgments: }

We thank Eleftherios Koutsofios for providing expertise on graph visualization
and helping with Graphviz. We are grateful to Simon Urbanek for computational
support, and Cheuk Yiu Ip and Lauro Lins for help with generating
visualizations in R.

\bibliographystyle{aaai}
\bibliography{flow_based}

\captionsetup[figure]{labelformat=empty}

\begin{figure}
\begin{centering}
\includegraphics[width=0.85\columnwidth]{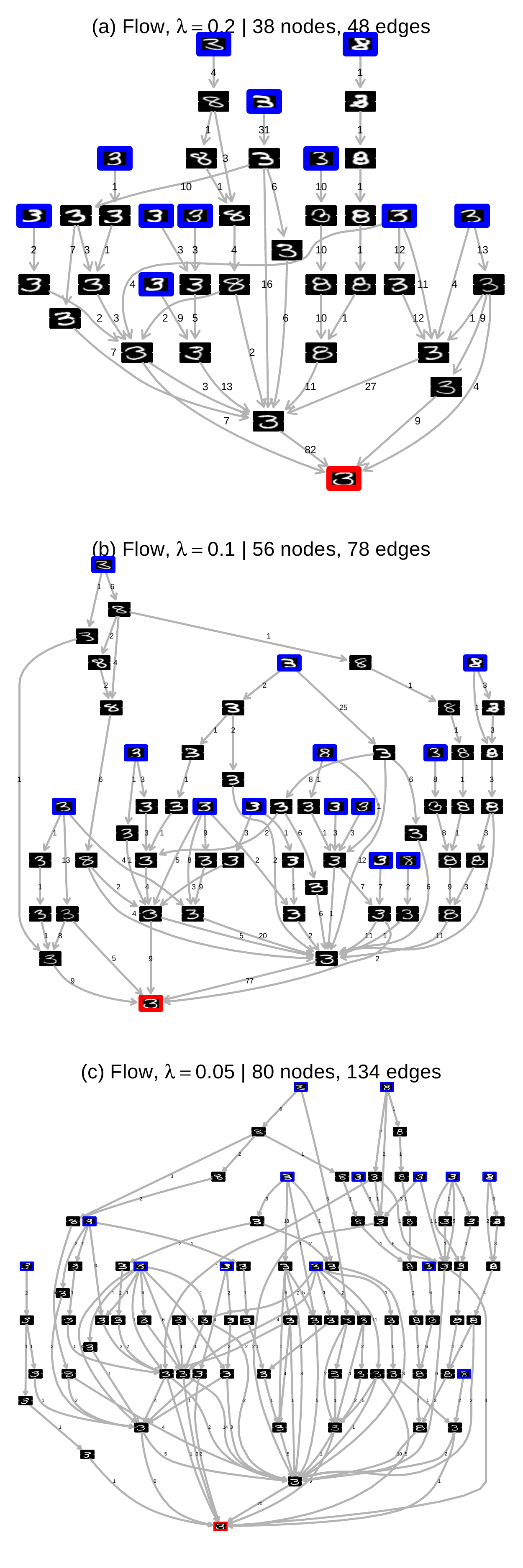} 
\par\end{centering}

\caption{Supplementary example 1}
\end{figure}

\begin{figure}
\begin{centering}
\includegraphics[width=0.85\columnwidth]{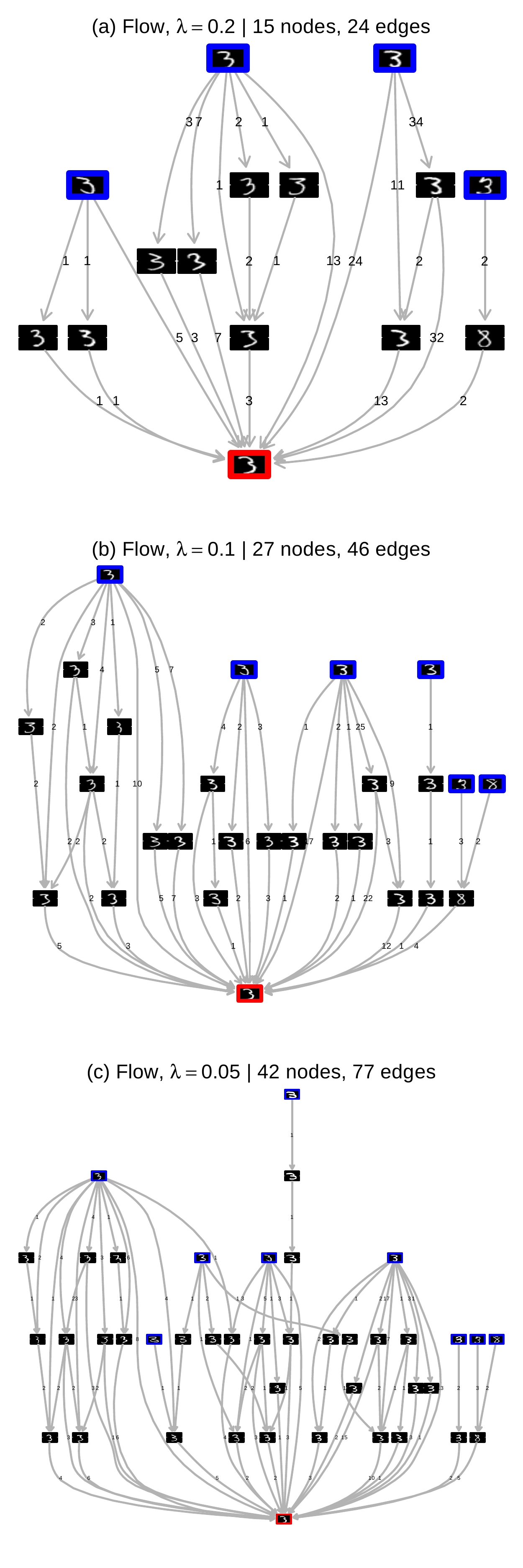} 
\par\end{centering}

\caption{Supplementary example 2}
\end{figure}

\begin{figure}
\begin{centering}
\includegraphics[width=0.85\columnwidth]{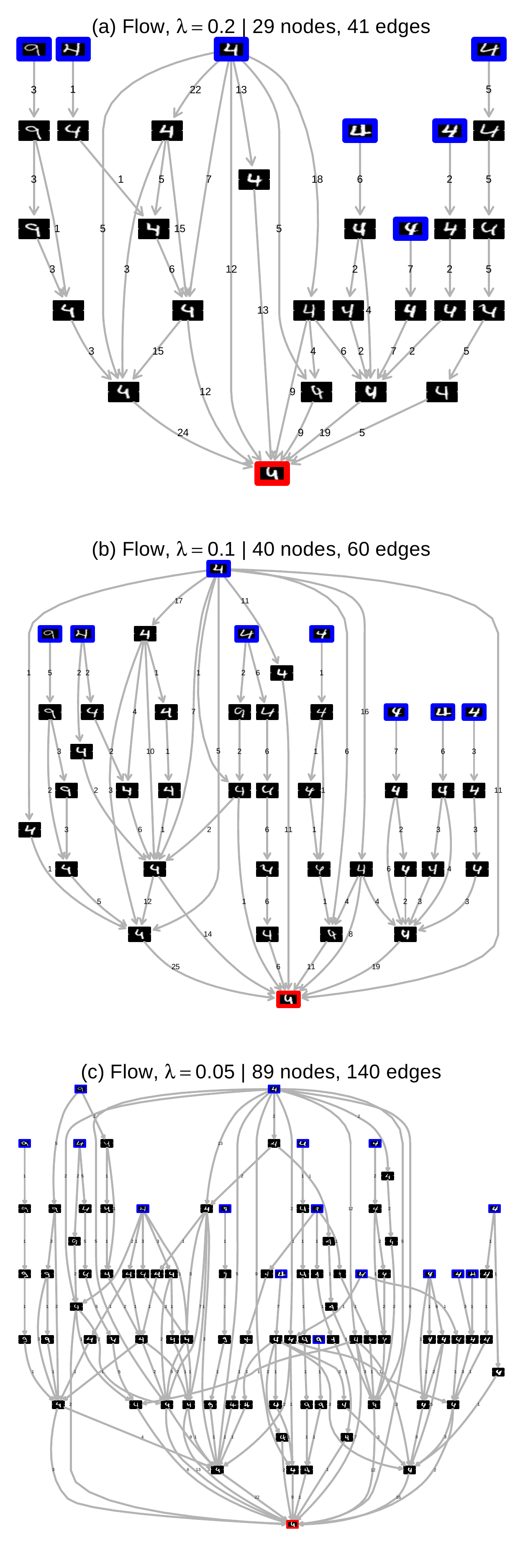} 
\par\end{centering}

\caption{Supplementary example 3}
\end{figure}

\begin{figure}
\begin{centering}
\includegraphics[width=0.85\columnwidth]{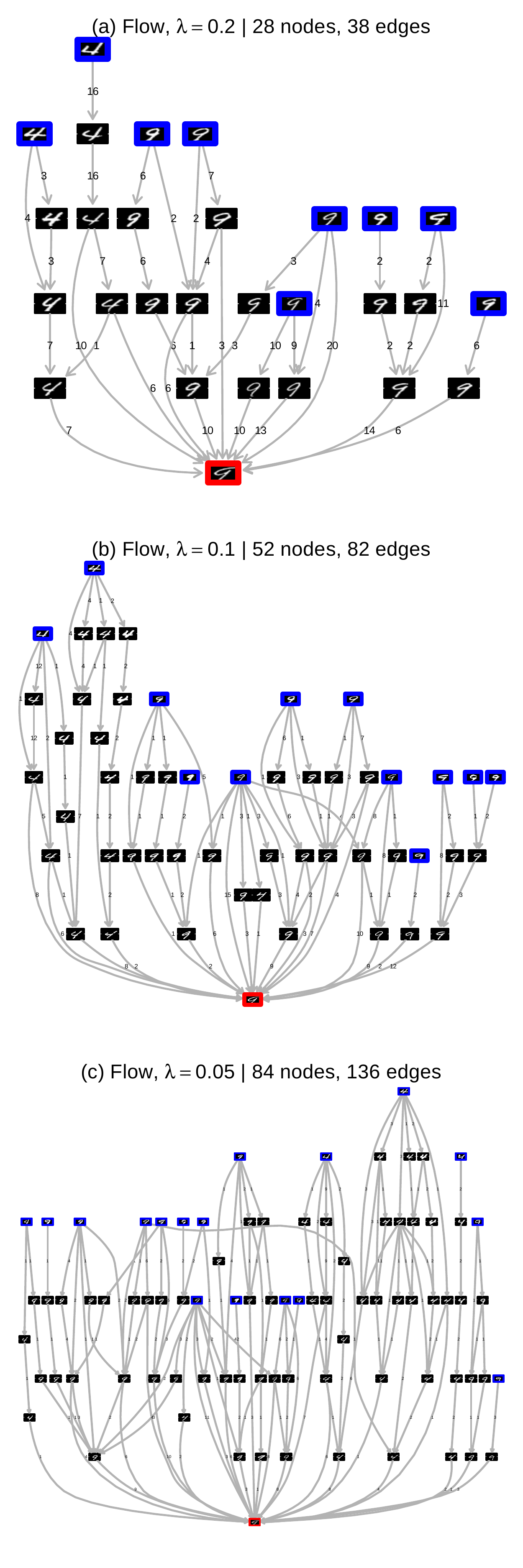} 
\par\end{centering}

\caption{Supplementary example 4}
\end{figure}

\end{document}